\documentclass{article} 
\usepackage{iclr2017_conference,times}
\usepackage{hyperref}
\usepackage{url}

\usepackage{bbm}
\usepackage{graphicx}
\usepackage{enumerate}
\usepackage{subcaption}
\usepackage{titlesec}
\usepackage{amsmath,amsfonts,amssymb}
\usepackage{bm}
\newcommand\transp{^\intercal\kern-\scriptspace}
\renewcommand{\vec}[1]{\mathbf{#1}}  
\newcommand*{\tran}{^{\mkern-1.5mu\mathsf{T}}}
\newcommand{\ud}{\mathrm{d}}

\usepackage{amsthm}
\usepackage{verbatim}
\newtheorem{theorem}{Theorem}
\newtheorem{lemma}{Lemma}

\newenvironment{tightitemize} 
{\vspace{-\topsep}\begin{itemize}\itemsep1pt \parskip0pt \parsep0pt \topsep1pt}
{\end{itemize}\vspace{-\topsep}}

\newenvironment{tightitemizeleft} 
{\vspace{-\topsep}\begin{itemize}[leftmargin=*]\itemsep1pt \parskip0pt \parsep0pt \topsep1pt}
{\end{itemize}\vspace{-\topsep}}
\usepackage{enumitem}

\title{Energy-based Generative Adversarial Networks}

\author{Junbo Zhao, Michael Mathieu and Yann LeCun \\
Department of Computer Science, New York University\\
Facebook Artificial Intelligence Research \\
\texttt{\{jakezhao, mathieu, yann\}@cs.nyu.edu}
}

\iclrfinalcopy 

\begin{document}
\maketitle

\begin{abstract}
We introduce the ``Energy-based Generative Adversarial Network'' model (EBGAN) which views the discriminator as an energy function that attributes low energies to the regions near the data manifold and higher energies to other regions. 
Similar to the probabilistic GANs, a generator is seen as being trained to produce contrastive samples with minimal energies, while the discriminator is trained to assign high energies to these generated samples. 
Viewing the discriminator as an energy function allows to use a wide variety of architectures and loss functionals in addition to the usual binary classifier with logistic output. 
Among them, we show one instantiation of EBGAN framework as using an auto-encoder architecture, with the energy being the reconstruction error, in place of the discriminator. We show that this form of EBGAN exhibits more stable behavior than regular GANs during training. We also show that a single-scale architecture can be trained to generate high-resolution images.
\end{abstract}

\section{Introduction}
\subsection{Energy-based model}

The essence of the energy-based model \citep{lecun2006tutorial} is to build a function that maps each point of an input space to a single scalar, which is called ``energy''. 
The learning phase is a data-driven process that shapes the energy surface in such a way that the desired configurations get assigned low energies, while the incorrect ones are given high energies. 
Supervised learning falls into this framework: for each $X$ in the training set, the energy of the pair $(X, Y)$ takes low values when $Y$ is the correct label and higher values for incorrect $Y$'s. 
Similarly, when modeling $X$ alone within an unsupervised learning setting, lower energy is attributed to the data manifold.
The term \emph{contrastive sample} is often used to refer to a data point causing an energy pull-up, such as the incorrect $Y$'s in supervised learning and points from low data density regions in unsupervised learning.

\subsection{Generative Adversarial Networks}
Generative Adversarial Networks (GAN) \citep{goodfellow2014generative} have led to significant improvements in image generation \citep{denton2015deep, radford2015unsupervised, im2016generating, salimans2016improved}, video prediction \citep{mathieu2015deep} and a number of other domains. 
The basic idea of GAN is to simultaneously train a discriminator and a generator. The discriminator is trained to distinguish \emph{real} samples of a dataset from \emph{fake} samples produced by the  generator. The generator uses input from an easy-to-sample random source, and is trained to produce fake samples that the discriminator cannot distinguish from real data samples. During training, the generator receives the gradient of the output of the discriminator with respect to the fake sample.
In the original formulation of GAN in \cite{goodfellow2014generative}, the discriminator produces a probability and, under certain conditions, convergence occurs when the distribution produced by the generator matches the data distribution.
From a game theory point of view, the convergence of a GAN is reached when the generator and the discriminator reach a Nash equilibrium.

\subsection{Energy-based Generative Adversarial Networks}
In this work, we propose to view the discriminator as an energy function (or a contrast function) without explicit probabilistic interpretation. The energy function computed by the discriminator can be viewed as a trainable cost function for the generator. The discriminator is trained to assign low energy values to the regions of high data density, and higher energy values outside these regions. Conversely, the generator can be viewed as a trainable parameterized function that produces samples in regions of the space to which the discriminator assigns low energy. While it is often possible to convert energies into probabilities through a Gibbs distribution \citep{lecun2006tutorial}, the absence of normalization in this energy-based form of GAN provides greater flexibility in the choice of architecture of the discriminator and the training procedure. 

The probabilistic binary discriminator in the original formulation of GAN can be seen as one way among many to define the contrast function and loss functional, as described in \cite{lecun2006tutorial} for the supervised and weakly supervised settings, and \cite{ranzato-unsup-07} for unsupervised learning. 
We experimentally demonstrate this concept, in the setting where the discriminator is an auto-encoder architecture, and the energy is the reconstruction error. More details of the interpretation of EBGAN are provided in the appendix \ref{app:more}.

Our main contributions are summarized as follows: 
\begin{tightitemize}
\item An energy-based formulation for generative adversarial training.
\item A proof that under a simple hinge loss, when the system reaches convergence, the generator of EBGAN produces points that follow the underlying data distribution.
\item An EBGAN framework with the discriminator using an auto-encoder architecture in which the energy is the reconstruction error.
\item A set of systematic experiments to explore hyper-parameters and architectural choices that produce good result for both EBGANs and probabilistic GANs.
\item A demonstration that EBGAN framework can be used to generate reasonable-looking high-resolution images from the ImageNet dataset at $256 \times 256$ pixel resolution, without a multi-scale approach.
\end{tightitemize}

\section{The EBGAN Model}
\label{sec:model}

Let $p_{data}$ be the underlying probability density of the distribution that produces the dataset.
The generator $G$ is trained to produce a sample $G(z)$, for instance an image, from a random vector $z$, which is sampled from a known distribution $p_z$, for instance $\mathcal{N}(0,1)$.
The discriminator $D$ takes either real or generated images, and estimates the energy value $E\in\mathbb{R}$ accordingly, as explained later. For simplicity, we assume that $D$ produces non-negative values, but the analysis would hold as long as the values are bounded below.

\subsection{Objective functional}
\label{sub:obj}

The output of the discriminator goes through an objective functional in order to shape the energy function, attributing low energy to the  real data samples and higher energy to the generated (``fake'') ones. In this work, we use a margin loss, but many other choices are possible as explained in~\cite{lecun2006tutorial}. Similarly to what has been done with the probabilistic GAN \citep{goodfellow2014generative}, we use a two different losses, one to train $D$ and the other to train $G$, in order to get better quality gradients when the generator is far from convergence. \\
Given a positive margin $m$, a data sample $x$ and a generated sample $G(z)$, the discriminator loss $\mathcal{L}_D$ and the generator loss $\mathcal{L}_G$ are formally defined by:
\begin{align}
\mathcal{L}_D(x, z) &= D(x) + \bm{[} m-D \big( G(z) \big) \bm{]} ^{\bm{+}} \label{discloss} \\
\label{genloss}\mathcal{L}_G(z) &= D \big( G(z) \big)
\end{align}
where $\bm{[\cdot]^{+}}=max(0,\cdot)$.
Minimizing $\mathcal{L}_G$ with respect to the parameters of $G$ is similar to maximizing the second term of $\mathcal{L}_D$. It has the same minimum but non-zero gradients when $D(G(z)) \geq m$.

\subsection{Optimality of the solution}
\label{sub:optim}

In this section, we present a theoretical analysis of the system presented in section~\ref{sub:obj}. We show that if the system reaches a Nash equilibrium, then the generator $G$ produces samples that are indistinguishable from the distribution of the dataset. This section is done in a non-parametric setting, \emph{i.e.} we assume that $D$ and $G$ have infinite capacity.

Given a generator $G$, let $p_G$ be the density distribution of $G(z)$ where $z \sim p_z$. In other words, $p_G$ is the density distribution of the samples generated by $G$. \\
We define $V(G,D)=\int_{x,z} \mathcal{L}_D(x,z) p_{data}(x) p_z(z)  \ud x\ud z$ and $U(G,D)=\int_{z} \mathcal{L}_G(z) p_z(z)\ud z$. We train the discriminator $D$ to minimize the quantity $V$ and the generator $G$ to minimize the quantity~$U$. \\
A Nash equilibrium of the system is a pair $(G^*, D^*)$ that satisfies:
\begin{eqnarray}
V(G^*, D^*) \leq V(G^*, D) & \quad & \forall D \label{nash:V} \\
U(G^*, D^*) \leq U(G, D^*) & \quad & \forall G \label{nash:U}
\end{eqnarray}

\begin{theorem}
\label{theo:nash}
If $(D^*, G^*)$ is a Nash equilibrium of the system, then $p_{G^*} = p_{data}$ almost everywhere, and $V(D^*, G^*) = m$.
\end{theorem}

\begin{proof}
First we observe that
\begin{eqnarray}
V(G^*, D) & = & \int_x p_{data}(x)D(x)\ud x + \int_z p_z(z) \left[m - D(G^*(z))\right]^+\ud z \\
& = & \int_x \left(p_{data}(x)D(x) + p_{G^*}(x)\left[m - D(x)\right]^+ \right)\ud x . \label{eqn:vgdef}
\end{eqnarray}
The analysis of the function $\varphi(y) = a y + b (m - y)^+$ (see lemma~\ref{lem:phi} in appendix~\ref{app:technical} for details) shows:\\
(a) $D^*(x) \leq m$ almost everywhere. To verify it, let us assume that there exists a set of measure non-zero such that $D^*(x) > m$. Let $\widetilde{D}(x) = \min(D^*(x), m)$. Then $V(G^*, \widetilde{D}) < V(G^*, D^*)$ which violates equation~\ref{nash:V}. \\
(b) The function $\varphi$ reaches its minimum in $m$ if $a<b$ and in $0$ otherwise. So $V(G^*, D)$ reaches its minimum when we replace $D^*(x)$ by these values. We obtain
\begin{eqnarray}
V(G^*, D^*) & = & m \int_x \mathbbm{1}_{p_{data}(x)<p_{G^*}(x)} p_{data}(x) \ud x + m \int_x \mathbbm{1}_{p_{data}(x)\geq p_{G^*}(x)} p_{G^*}(x) \ud x \\
& = & m \int_x \left( \mathbbm{1}_{p_{data}(x)<p_{G^*}(x)} p_{data}(x) + \left(1 - \mathbbm{1}_{p_{data}(x)< p_{G^*}(x)}\right) p_{G^*}(x) \right) \ud x \\
& = & m \int_x p_{G^*}(x) \ud x + m \int_x \mathbbm{1}_{p_{data}(x)< p_{G^*}(x)} (p_{data}(x) - p_{G^*}(x)) \ud x \\
& = & m + m \int_x \mathbbm{1}_{p_{data}(x)< p_{G^*}(x)} (p_{data}(x) - p_{G^*}(x)) \ud x .  \label{eqn:nonpos}
\end{eqnarray}
The second term in equation~\ref{eqn:nonpos} is non-positive, so $V(G^*, D^*) \leq m$.

By putting the ideal generator that generates $p_{data}$ into the right side of equation~\ref{nash:U}, we get
\begin{eqnarray}
&\displaystyle\int_x p_{G^*}(x)D^*(x)\ud x \leq \int_x p_{data}(x) D^*(x)\ud x .\\
\text{Thus by (\ref{eqn:vgdef}),\quad} &\displaystyle\int_x p_{G^*}(x)D^*(x)\ud x + \int_x p_{G^*}(x)[m-D^*(x)]^+\ud x
\leq V(G^*, D^*)
\end{eqnarray}
and since $D^*(x) \leq m$, we get $m \leq V(G^*, D^*)$.

Thus, $m\leq V(G^*, D^*) \leq m$ \emph{i.e.} $V(G^*, D^*)=m$. Using equation~\ref{eqn:nonpos}, we see that can only happen if $\int_x \mathbbm{1}_{p_{data}(x)< p_G(x)} \ud x = 0$, which is true if and only if $p_G = p_{data}$ almost everywhere (this is because $p_{data}$ and $p_G$ are probabilities densities, see lemma~\ref{lem:indicator} in the appendix~\ref{app:technical} for details).
\end{proof}

\begin{theorem}
\label{theo:charac}
A Nash equilibrium of this system exists and is characterized by (a) $p_{G^*}=p_{data}$ (almost everywhere) and (b) there exists a constant $\gamma \in [0,m]$ such that $D^*(x)=\gamma$ (almost everywhere).\footnote{This is assuming there is no region where $p_{data}(x)=0$. If such a region exists, $D^*(x)$ may have any value in $[0,m]$ for $x$ in this region.}.
\end{theorem}

\begin{proof}
See appendix \ref{app:technical}.
\end{proof}

\subsection{Using auto-encoders}
In our experiments, the discriminator $D$ is structured as an auto-encoder:
\begin{equation}
D(x) = ||Dec(Enc(x)) - x||.
\end{equation}

\begin{figure}[ht]
  \centering
  \includegraphics[width=1\linewidth]{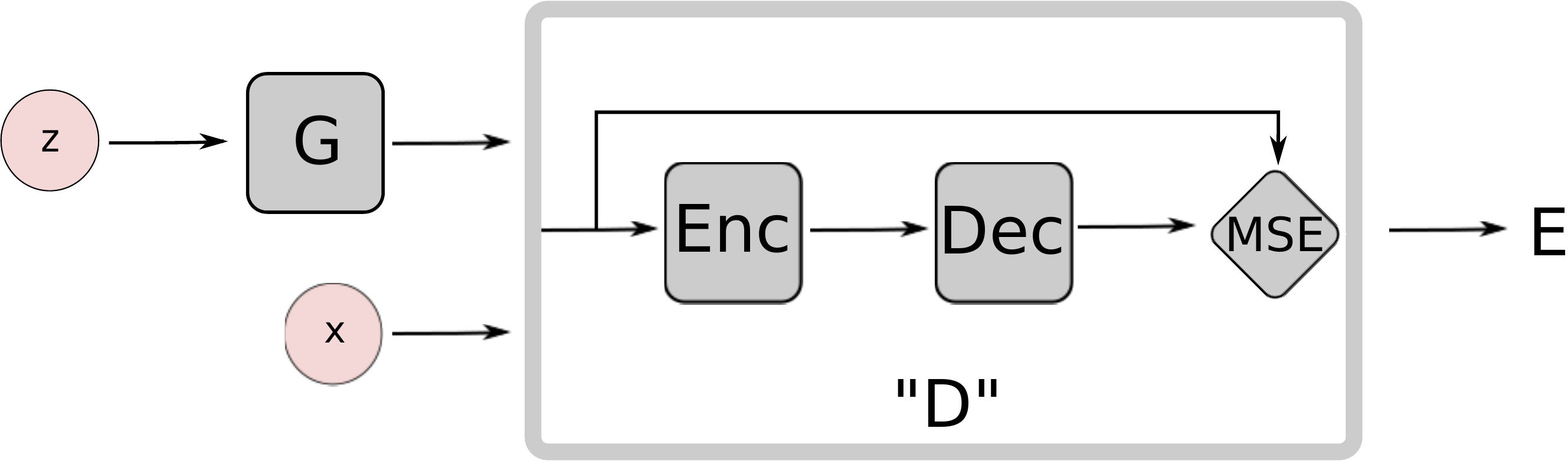}
  \caption{\label{fig:model} EBGAN architecture with an auto-encoder discriminator.}  
\end{figure}

The diagram of the EBGAN model with an auto-encoder discriminator is depicted in figure \ref{fig:model}.
The choice of the auto-encoders for $D$ may seem arbitrary at the first glance, yet we postulate that it is conceptually more attractive than a binary logistic network:
\begin{tightitemizeleft}
\item Rather than using a single bit of target information to train the model, the reconstruction-based output offers a diverse targets for the discriminator. With the binary logistic loss, only two targets are possible, so within a minibatch, the gradients corresponding to different samples are most likely far from orthogonal. This leads to inefficient training, and reducing the minibatch sizes is often not an option on current hardware. On the other hand, the reconstruction loss will likely produce very different gradient directions within the minibatch, allowing for larger minibatch size without loss of efficiency.

\item  Auto-encoders have traditionally been used to represent energy-based model and arise naturally. 
When trained with some regularization terms (see section \ref{subsub:conn}), auto-encoders have the ability to learn an energy manifold without supervision or negative examples.
This means that even when an EBGAN auto-encoding model is trained to reconstruct a \emph{real} sample, the discriminator contributes to discovering the data manifold by itself. 
To the contrary, without the presence of negative examples from the generator, a discriminator trained with binary logistic loss becomes pointless.
\end{tightitemizeleft}

\subsubsection{Connection to the regularized auto-encoders}
\label{subsub:conn}
One common issue in training auto-encoders is that the model may learn little more than an identity function, meaning that it attributes zero energy to the whole space. In order to avoid this problem, the model must be pushed to give higher energy to points outside the data manifold.
Theoretical and experimental results have addressed this issue by regularizing the latent representations \citep{vincent2010stacked, rifai2011contractive, marc2007efficient, kavukcuoglu2010learning}.
Such regularizers aim at restricting the reconstructing power of the auto-encoder so that it can only attribute low energy to a smaller portion of the input points.

We argue that the energy function (the discriminator) in the EBGAN framework is also seen as being regularized by having a generator producing the contrastive samples, to which the discriminator ought to give high reconstruction energies.
We further argue that the EBGAN framework allows more flexibility from this perspective, because: (i)-the regularizer (generator) is fully trainable instead of being handcrafted; (ii)-the adversarial training paradigm enables a direct interaction between the duality of producing contrastive sample and learning the energy function.

\subsection{Repelling regularizer}
We propose a ``repelling regularizer'' which fits well into the EBGAN auto-encoder model, purposely keeping the model from producing samples that are clustered in one or only few modes of $p_{data}$. 
Another technique ``minibatch discrimination'' was developed by \cite{salimans2016improved} from the same philosophy.

Implementing the repelling regularizer involves a Pulling-away Term (PT) that runs at a representation level. Formally, let $S \in \mathbb{R}^{s \times N}$ denotes a batch of sample representations taken from the encoder output layer. Let us define PT as: 
\begin{equation}
f_{PT}(S) = \frac{1}{N (N-1)} \sum_{i} \sum_{j \neq i} \Big( \frac{S_i \tran S_j}{ \|S_i\| \|S_j\| } \Big) ^2.
\end{equation}
PT operates on a mini-batch and attempts to orthogonalize the pairwise sample representation. It is inspired by the prior work showing the representational power of the encoder in the auto-encoder alike model such as \cite{rasmus2015semi} and \cite{zhao2015stacked}.
The rationale for choosing the cosine similarity instead of Euclidean distance is to make the term bounded below and invariant to scale.
We use the notation ``EBGAN-PT'' to refer to the EBGAN auto-encoder model trained with this term.
Note the PT is used in the generator loss but not in the discriminator loss.

\section{Related work}
Our work primarily casts GANs into an energy-based model scope.
On this direction, the approaches studying contrastive samples are relevant to EBGAN, such as the use of noisy samples \citep{vincent2010stacked} and noisy gradient descent methods like contrastive divergence \citep{carreira2005contrastive}.
From the perspective of GANs, several papers were presented to improve the stability of GAN training, \citep{salimans2016improved, denton2015deep, radford2015unsupervised, im2016generating, mathieu2015deep}.

\cite{kim2016deep} propose a probabilistic GAN and cast it into an energy-based density estimator by using the Gibbs distribution. Quite unlike EBGAN, this proposed framework doesn't get rid of the computational challenging partition function, so the choice of the energy function is required to be integratable. 

\section{Experiments}
\subsection{Exhaustive grid search on MNIST}
\label{sub:grid}
In this section, we study the training stability of EBGANs over GANs on a simple task of MNIST digit generation with fully-connected networks. We run an exhaustive grid search over a set of architectural choices and hyper-parameters for both frameworks.

Formally, we specify the search grid in table \ref{tab:grid}. 
We impose the following restrictions on EBGAN models: (i)-using learning rate 0.001 and Adam \citep{kingma2014adam} for both $G$ and $D$; (ii)-\texttt{nLayerD} represents the total number of layers combining $Enc$ and $Dec$. For simplicity, we fix $Dec$ to be one layer and only tune the $Enc$ \#layers; (iii)-the margin is set to 10 and not being tuned.
To analyze the results, we use the \emph{inception score} \citep{salimans2016improved} as a numerical means reflecting the generation quality. Some slight modification of the formulation were made to make figure \ref{fig:hist_all} visually more approachable while maintaining the score's original meaning, $I' = E_x KL(p(y)||p(y|\vec{x}))$\footnote{This form of the ``inception score'' is only used to better analyze the grid search in the scope of this work, but not to compare with any other published work.} (more details in appendix \ref{app:grid}). Briefly, higher $I'$ score implies better generation quality.

\begin{table}[h]
\small
\vspace{-6mm}
\caption{Grid search specs}
\label{tab:grid}
\begin{center}
\vspace{-3mm}
\begin{tabular}{llcc}
\multicolumn{1}{l}{\bf Settings} & \multicolumn{1}{c}{\bf Description} &\multicolumn{1}{c}{\bf EBGANs} &\multicolumn{1}{c}{\bf GANs}
\\ \hline \\
\texttt{nLayerG}   & number of layers in $G$       &[2, 3, 4, 5]              &[2, 3, 4, 5] \\
\texttt{nLayerD}   & number of layers in $D$ &[2, 3, 4, 5]              &[2, 3, 4, 5] \\
\texttt{sizeG}     & number of neurons in $G$     &[400, 800, 1600, 3200]    & [400, 800, 1600, 3200]  \\
\texttt{sizeD}     & number of neurons in $D$      &[128, 256, 512, 1024]     & [128, 256, 512, 1024] \\
\texttt{dropoutD}  & if to use dropout in $D$      &[true, false]             & [true, false] \\
\texttt{optimD}    & to use Adam or SGD for $D$      &adam                      & [adam, sgd] \\
\texttt{optimG}    & to use Adam or SGD for $G$     &adam                      & [adam, sgd]\\
\texttt{lr}        & learning rate      &0.001                     & [0.01, 0.001, 0.0001]\\
\#experiments: & - & {\bf 512}            & {\bf 6144}
\end{tabular}
\end{center}
\end{table}

{\bf Histograms} We plot the histogram of $I'$ scores in figure \ref{fig:hist_all}. 
We further separated out the optimization related setting from GAN's grid (\texttt{optimD}, \texttt{optimG} and \texttt{lr}) and plot the histogram of each sub-grid individually, together with the EBGAN $I'$ scores as a reference, in figure \ref{fig:hist_sep_all}. The number of experiments for GANs and EBGANs are both 512 in every subplot.
The histograms evidently show that EBGANs are more reliably trained. 

Digits generated from the configurations presenting the best inception score are shown in figure \ref{fig:mnist_gen}.

\begin{figure}[h]
  \centering
  \includegraphics[width=1\linewidth]{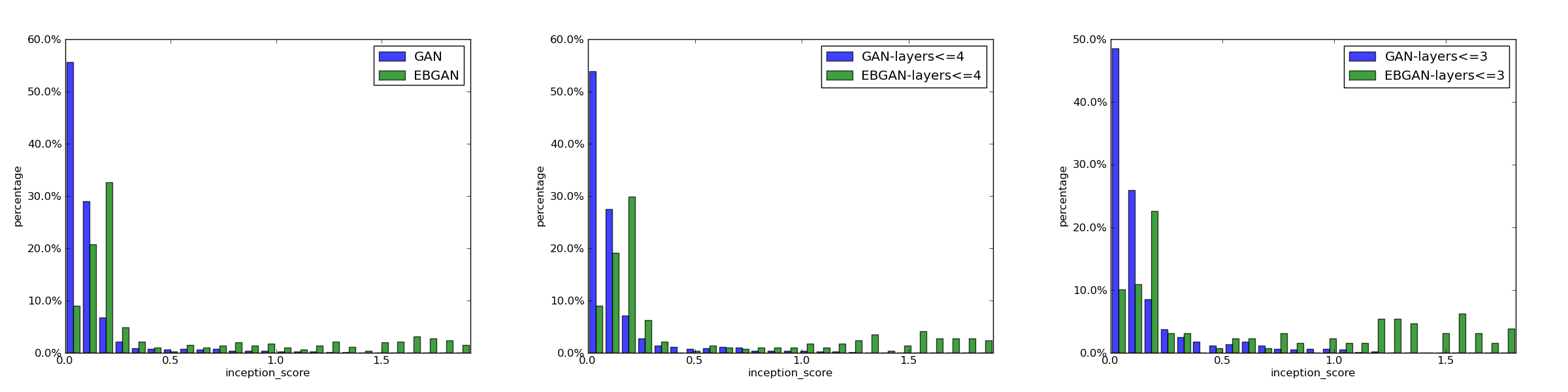}
  \caption{\label{fig:hist_all} {\bf (Zooming in on pdf file is recommended.)} Histogram of the inception scores from the grid search. The x-axis carries the inception score $I$ and y-axis informs the portion of the models (in percentage) falling into certain bins.
Left (a): general comparison of EBGANs against GANs; Middle (b): EBGANs and GANs both constrained by \texttt{nLayer[GD]<=4}; Right (c): EBGANs and GANs both constrained by \texttt{nLayer[GD]<=3}.
}
\end{figure}

\begin{figure}[h]
  \centering
  \vspace{-3mm}
  \includegraphics[width=1\linewidth]{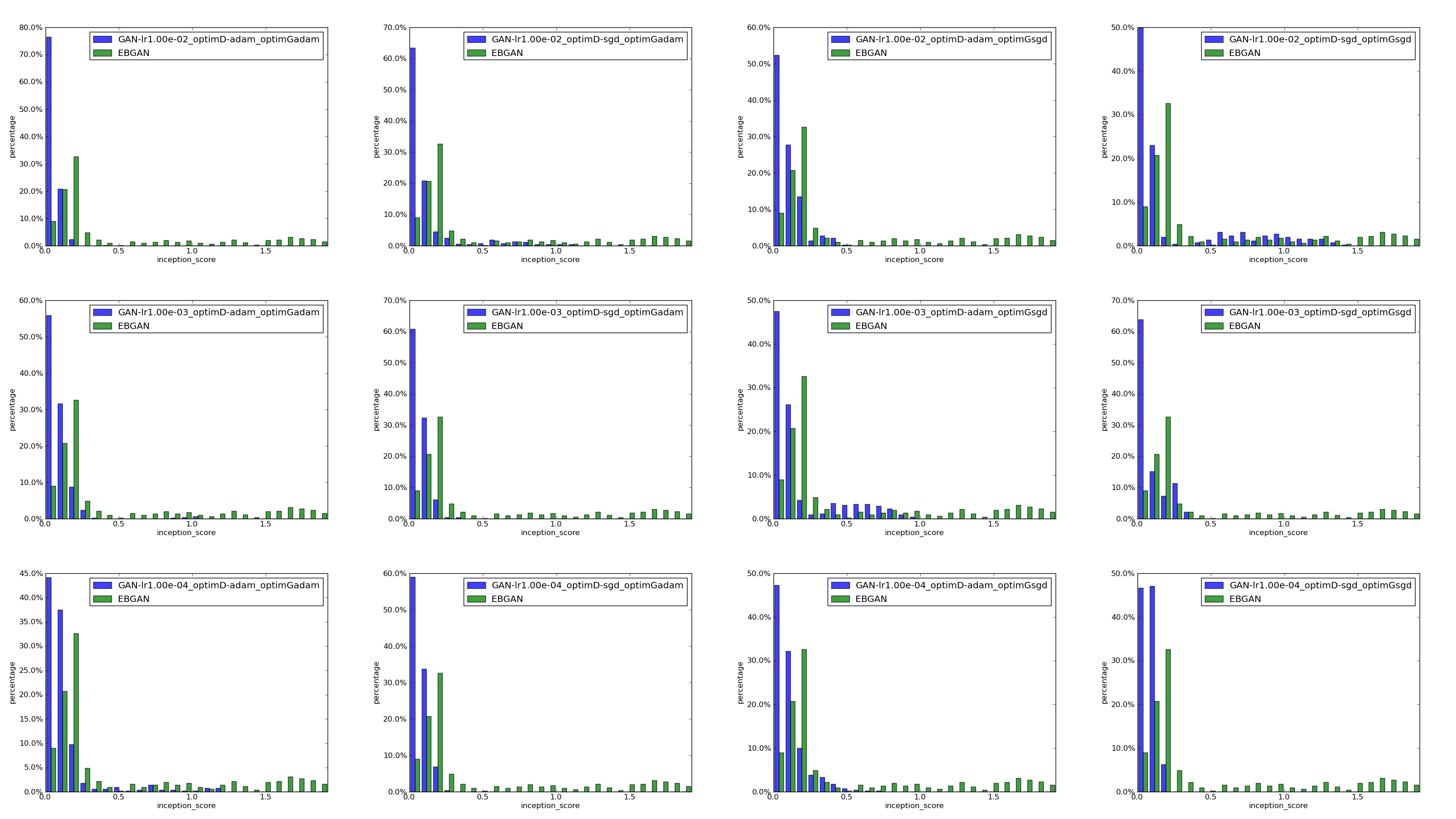}
  \caption{\label{fig:hist_sep_all} {\bf (Zooming in on pdf file is recommended.)} Histogram of the inception scores grouped by different optimization combinations, drawn from \texttt{optimD}, \texttt{optimG} and \texttt{lr} (See text). 
}
\end{figure}

\begin{figure}[ht]
\centering
\minipage{0.30\textwidth}
\includegraphics[width=\linewidth]{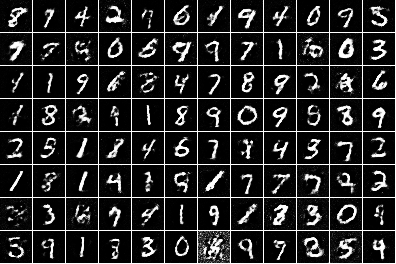}
\endminipage \hspace{5pt}
\minipage{0.3\textwidth}
\includegraphics[width=\linewidth]{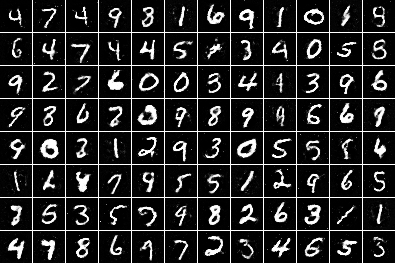}
\endminipage \hspace{5pt}
\minipage{0.3\textwidth}
\includegraphics[width=\linewidth]{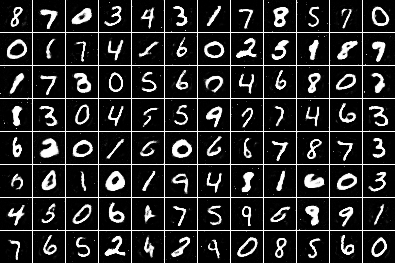}
\endminipage
\caption{Generation from the grid search on MNIST. Left(a): Best GAN model; Middle(b): Best EBGAN model. Right(c): Best EBGAN-PT model.}
\label{fig:mnist_gen}
\end{figure}

\subsection{Semi-supervised learning on MNIST}
\label{sub:semi}
We explore the potential of using the EBGAN framework for semi-supervised learning on permutation-invariant MNIST, collectively on using 100, 200 and 1000 labels.
We utilized a bottom-layer-cost Ladder Network (LN) \citep{rasmus2015semi} with the EGBAN framework (EBGAN-LN). Ladder Network can be categorized as an energy-based model that is built with both feedforward and feedback hierarchies powered by stage-wise lateral connections coupling two pathways.

One technique we found crucial in enabling EBGAN framework for semi-supervised learning is to gradually decay the margin value $m$ of the equation \ref{discloss}. The rationale behind is to let discriminator punish generator less when $p_G$ gets closer to the data manifold. One can think of the extreme case where the contrastive samples are exactly pinned on the data manifold, such that they are ``not contrastive anymore''. This ultimate status happens when $m=0$ and the EBGAN-LN model falls back to a normal Ladder Network. The undesirability of a non-decay dynamics for using the discriminator in the GAN or EBGAN framework is also indicated by Theorem \ref{theo:charac}: on convergence, the discriminator reflects a flat energy surface. However, we posit that the trajectory of learning a EBGAN-LN model does provide the LN (discriminator) more information by letting it see contrastive samples. Yet the optimal way to avoid the mentioned undesirability is to make sure $m$ has been decayed to $0$ when the Nash Equilibrium is reached.  The margin decaying schedule is found by hyper-parameter search in our experiments (technical details in appendix \ref{app:semi}).

From table \ref{tab:semi}, it shows that positioning a bottom-layer-cost LN into an EBGAN framework profitably improves the performance of the LN itself.
We postulate that within the scope of the EBGAN framework, iteratively feeding the adversarial contrastive samples produced by the generator to the energy function acts as an effective regularizer; the contrastive samples can be thought as an extension to the dataset that provides more information to the classifier.
We notice there was a discrepancy between the reported results between \cite{rasmus2015semi} and \cite{pezeshki2015deconstructing}, so we report both results along with our own implementation of the Ladder Network running the same setting.
The specific experimental setting and analysis are available in appendix \ref{app:semi}.
\begin{table}[h]
\small
\caption{The comparison of LN bottom-layer-cost model and its EBGAN extension on PI-MNIST semi-supervised task. Note the results are error rate (in \%) and averaged over 15 different random seeds.}
\label{tab:semi}
\begin{center}
\vspace{-6mm}
\begin{tabular}{lccc}
\multicolumn{1}{c}{\bf model}  &\multicolumn{1}{c}{\bf 100}  &\multicolumn{1}{c}{\bf 200}  &\multicolumn{1}{c}{\bf 1000}
\\ \hline \noalign{\vskip 3pt}  
LN bottom-layer-cost, reported in \cite{pezeshki2015deconstructing} & 1.69$\pm$0.18&-&1.05$\pm$0.02\\
LN bottom-layer-cost, reported in \cite{rasmus2015semi}        & 1.09$\pm$0.32 &-& 0.90$\pm$0.05\\
\hline \noalign{\vskip 1.5pt}  
LN bottom-layer-cost, reproduced in this work (see appendix \ref{app:semi}) & 1.36$\pm$0.21 & 1.24$\pm$0.09 & 1.04$\pm$0.06 \\
LN bottom-layer-cost within EBGAN framework   & \textbf{1.04$\pm$0.12} & \textbf{0.99$\pm$0.12} & \textbf{0.89$\pm$0.04} \\
Relative percentage improvement & 23.5\% & 20.2\% & 14.4\%
\\ \hline
\end{tabular}
\end{center}
\end{table}

\subsection{LSUN \& CelebA}
\label{sub:lsun}

\begin{figure}[ht]
\centering
\minipage{0.48\textwidth}
\includegraphics[width=\linewidth]{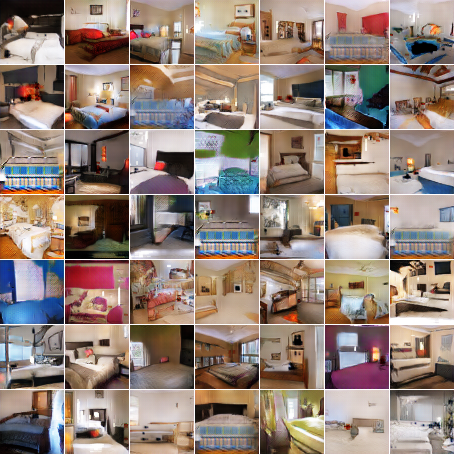}
\endminipage \hspace{5pt}
\minipage{0.48\textwidth}
\includegraphics[width=\linewidth]{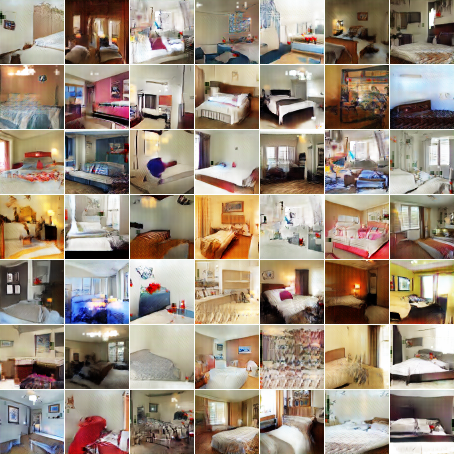}
\endminipage
\caption{Generation from the LSUN bedroom dataset. Left(a): DCGAN generation. Right(b): EBGAN-PT generation.}
\label{fig:lsun_full_gen}
\end{figure}

\begin{figure}[ht]
\centering
\minipage{0.48\textwidth}
\includegraphics[width=\linewidth]{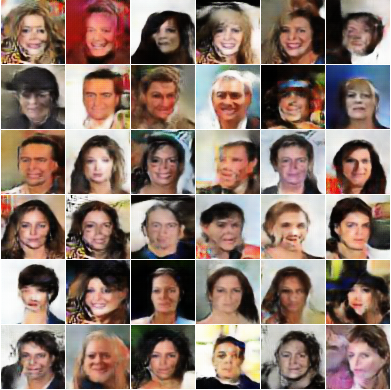}
\endminipage \hspace{5pt}
\minipage{0.48\textwidth}
\includegraphics[width=\linewidth]{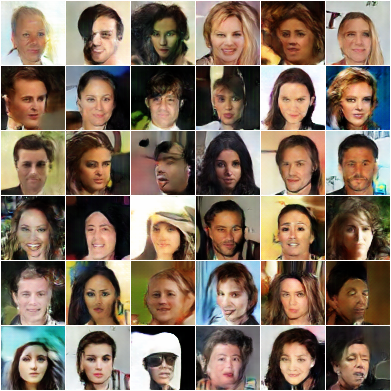}
\endminipage
\caption{Generation from the CelebA dataset. Left(a): DCGAN generation. Right(b): EBGAN-PT generation.}
\label{fig:celeba_gen}
\end{figure}
We apply the EBGAN framework with deep convolutional architecture to generate $64 \times 64$ RGB images, a more realistic task, using the LSUN bedroom dataset \citep{yu2015lsun} and the large-scale face dataset CelebA under alignment \citep{liu2015deep}.
To compare EBGANs with DCGANs \citep{radford2015unsupervised}, we train a DCGAN model under the same configuration and show its generation side-by-side with the EBGAN model, in figures \ref{fig:lsun_full_gen} and \ref{fig:celeba_gen}. The specific settings are listed in appendix \ref{app:lsun}.

\subsection{ImageNet}
\label{sub:imagenet}

\begin{figure}[ht]
  \centering
  \vspace{-3mm}
  \includegraphics[width=1\linewidth]{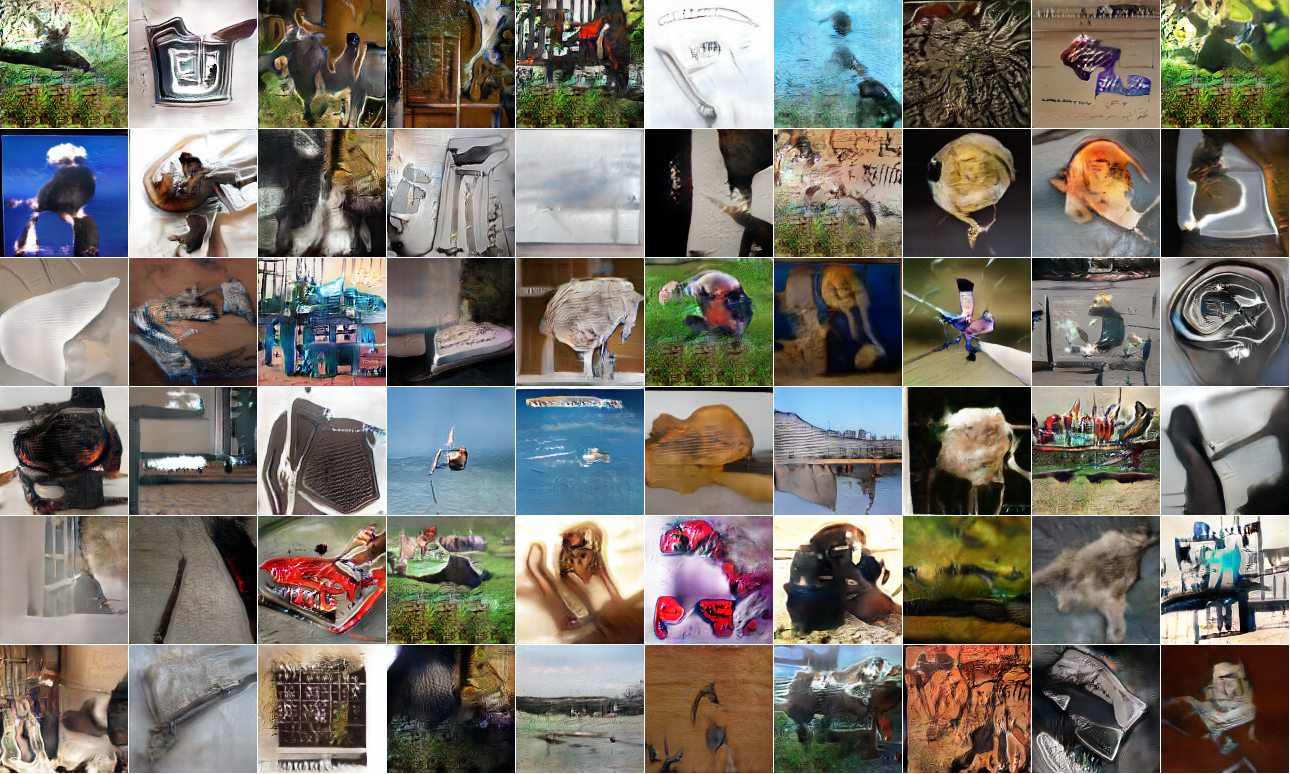}
  \caption{\label{fig:imagenet_128} ImageNet $128 \times 128$ generations using an EBGAN-PT.}
  \vspace{-2mm}
\end{figure}

\begin{figure}[ht]
  \centering
  \vspace{-1mm}
  \includegraphics[width=1\linewidth]{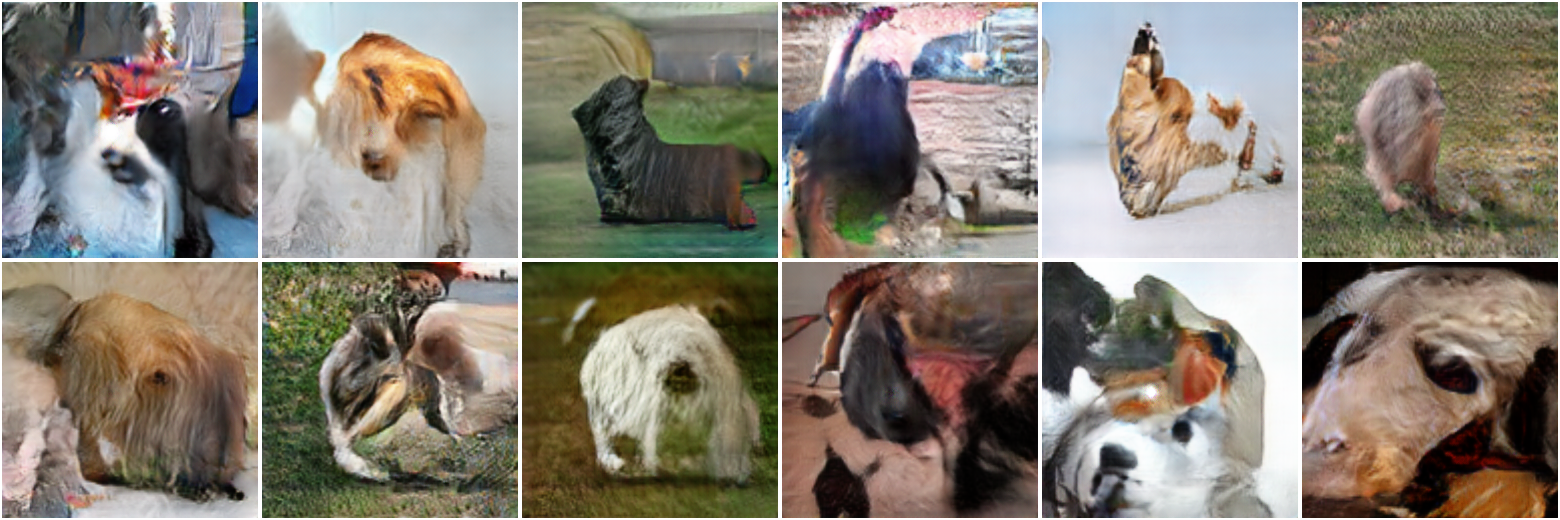}
  \caption{\label{fig:imagenet_256} ImageNet $256 \times 256$ generations using an EBGAN-PT.}
  \vspace{-2mm}
\end{figure}

Finally, we trained EBGANs to generate high-resolution images on ImageNet \citep{ILSVRC15}.
Compared with the datasets we have experimented so far, ImageNet presents an extensively larger and wilder space, so modeling the data distribution by a generative model becomes very challenging.
We devised an experiment to generate $128 \times 128$ images, trained on the full ImageNet-1k dataset, which contains roughly 1.3 million images from 1000 different categories. 
We also trained a network to generate images of size $256 \times 256$, on a dog-breed subset of ImageNet, using the wordNet IDs provided by \cite{vinyals2016matching}.
The results are shown in figures \ref{fig:imagenet_128} and \ref{fig:imagenet_256}.
Despite the difficulty of generating images on a high-resolution level, we observe that EBGANs are able to learn about the fact that objects appear in the foreground, together with various background components resembling grass texture, sea under the horizon, mirrored mountain in the water, buildings, etc. 
In addition, our $256 \times 256$ dog-breed generations, although far from realistic, do reflect some  knowledge about the appearances of dogs such as their body, furs and eye. 

\section{Outlook}
We bridge two classes of unsupervised learning methods -- GANs and auto-encoders -- and revisit the GAN framework from an alternative energy-based perspective. EBGANs show better convergence pattern and scalability to generate high-resolution images. 
A family of energy-based loss functionals presented in \cite{lecun2006tutorial} can easily be incorporated into the EBGAN framework. 
For the future work, the conditional setting \citep{denton2015deep,mathieu2015deep} is a promising setup to explore.
We hope the future research will raise more attention on a broader view of GANs from the energy-based perspective.

\vspace{-2mm}
\subsection*{Acknowledgment}
\vspace{-1mm}
We thank Emily Denton, Soumith Chitala, Arthur Szlam, Marc'Aurelio Ranzato, Pablo Sprechmann, Ross Goroshin and Ruoyu Sun for fruitful discussions.
We also thank Emily Denton and Tian Jiang for their help with the manuscript.

{\small
\bibliography{iclr2017_conference}

\begin{thebibliography}{23}
\providecommand{\natexlab}[1]{#1}
\providecommand{\url}[1]{\texttt{#1}}
\expandafter\ifx\csname urlstyle\endcsname\relax
  \providecommand{\doi}[1]{doi: #1}\else
  \providecommand{\doi}{doi: \begingroup \urlstyle{rm}\Url}\fi

\bibitem[Carreira-Perpinan \& Hinton(2005)Carreira-Perpinan and
  Hinton]{carreira2005contrastive}
Carreira-Perpinan, Miguel~A and Hinton, Geoffrey.
\newblock On contrastive divergence learning.
\newblock In \emph{AISTATS}, volume~10, pp.\  33--40. Citeseer, 2005.

\bibitem[Denton et~al.(2015)Denton, Chintala, Fergus, et~al.]{denton2015deep}
Denton, Emily~L, Chintala, Soumith, Fergus, Rob, et~al.
\newblock Deep generative image models using a laplacian pyramid of adversarial
  networks.
\newblock In \emph{Advances in neural information processing systems}, pp.\
  1486--1494, 2015.

\bibitem[Goodfellow et~al.(2014)Goodfellow, Pouget-Abadie, Mirza, Xu,
  Warde-Farley, Ozair, Courville, and Bengio]{goodfellow2014generative}
Goodfellow, Ian, Pouget-Abadie, Jean, Mirza, Mehdi, Xu, Bing, Warde-Farley,
  David, Ozair, Sherjil, Courville, Aaron, and Bengio, Yoshua.
\newblock Generative adversarial nets.
\newblock In \emph{Advances in Neural Information Processing Systems}, pp.\
  2672--2680, 2014.

\bibitem[Im et~al.(2016)Im, Kim, Jiang, and Memisevic]{im2016generating}
Im, Daniel~Jiwoong, Kim, Chris~Dongjoo, Jiang, Hui, and Memisevic, Roland.
\newblock Generating images with recurrent adversarial networks.
\newblock \emph{arXiv preprint arXiv:1602.05110}, 2016.

\bibitem[Ioffe \& Szegedy(2015)Ioffe and Szegedy]{ioffe2015batch}
Ioffe, Sergey and Szegedy, Christian.
\newblock Batch normalization: Accelerating deep network training by reducing
  internal covariate shift.
\newblock \emph{arXiv preprint arXiv:1502.03167}, 2015.

\bibitem[Kavukcuoglu et~al.(2010)Kavukcuoglu, Sermanet, Boureau, Gregor,
  Mathieu, and Cun]{kavukcuoglu2010learning}
Kavukcuoglu, Koray, Sermanet, Pierre, Boureau, Y-Lan, Gregor, Karol, Mathieu,
  Micha{\"e}l, and Cun, Yann~L.
\newblock Learning convolutional feature hierarchies for visual recognition.
\newblock In \emph{Advances in neural information processing systems}, pp.\
  1090--1098, 2010.

\bibitem[Kim \& Bengio(2016)Kim and Bengio]{kim2016deep}
Kim, Taesup and Bengio, Yoshua.
\newblock Deep directed generative models with energy-based probability
  estimation.
\newblock \emph{arXiv preprint arXiv:1606.03439}, 2016.

\bibitem[Kingma \& Ba(2014)Kingma and Ba]{kingma2014adam}
Kingma, Diederik and Ba, Jimmy.
\newblock Adam: A method for stochastic optimization.
\newblock \emph{arXiv preprint arXiv:1412.6980}, 2014.

\bibitem[LeCun et~al.(2006)LeCun, Chopra, and Hadsell]{lecun2006tutorial}
LeCun, Yann, Chopra, Sumit, and Hadsell, Raia.
\newblock A tutorial on energy-based learning.
\newblock 2006.

\bibitem[Liu et~al.(2015)Liu, Luo, Wang, and Tang]{liu2015deep}
Liu, Ziwei, Luo, Ping, Wang, Xiaogang, and Tang, Xiaoou.
\newblock Deep learning face attributes in the wild.
\newblock In \emph{Proceedings of the IEEE International Conference on Computer
  Vision}, pp.\  3730--3738, 2015.

\bibitem[Marc’Aurelio~Ranzato \& Chopra(2007)Marc’Aurelio~Ranzato and
  Chopra]{marc2007efficient}
Marc’Aurelio~Ranzato, Christopher~Poultney and Chopra, Sumit.
\newblock Efficient learning of sparse representations with an energy-based
  model.
\newblock 2007.

\bibitem[Mathieu et~al.(2015)Mathieu, Couprie, and LeCun]{mathieu2015deep}
Mathieu, Michael, Couprie, Camille, and LeCun, Yann.
\newblock Deep multi-scale video prediction beyond mean square error.
\newblock \emph{arXiv preprint arXiv:1511.05440}, 2015.

\bibitem[Pezeshki et~al.(2015)Pezeshki, Fan, Brakel, Courville, and
  Bengio]{pezeshki2015deconstructing}
Pezeshki, Mohammad, Fan, Linxi, Brakel, Philemon, Courville, Aaron, and Bengio,
  Yoshua.
\newblock Deconstructing the ladder network architecture.
\newblock \emph{arXiv preprint arXiv:1511.06430}, 2015.

\bibitem[Radford et~al.(2015)Radford, Metz, and
  Chintala]{radford2015unsupervised}
Radford, Alec, Metz, Luke, and Chintala, Soumith.
\newblock Unsupervised representation learning with deep convolutional
  generative adversarial networks.
\newblock \emph{arXiv preprint arXiv:1511.06434}, 2015.

\bibitem[Ranzato et~al.(2007)Ranzato, Boureau, Chopra, and
  LeCun]{ranzato-unsup-07}
Ranzato, {Marc'Aurelio}, Boureau, {Y-Lan}, Chopra, Sumit, and LeCun, Yann.
\newblock A unified energy-based framework for unsupervised learning.
\newblock In \emph{Proc. Conference on AI and Statistics (AI-Stats)}, 2007.

\bibitem[Rasmus et~al.(2015)Rasmus, Berglund, Honkala, Valpola, and
  Raiko]{rasmus2015semi}
Rasmus, Antti, Berglund, Mathias, Honkala, Mikko, Valpola, Harri, and Raiko,
  Tapani.
\newblock Semi-supervised learning with ladder networks.
\newblock In \emph{Advances in Neural Information Processing Systems}, pp.\
  3546--3554, 2015.

\bibitem[Rifai et~al.(2011)Rifai, Vincent, Muller, Glorot, and
  Bengio]{rifai2011contractive}
Rifai, Salah, Vincent, Pascal, Muller, Xavier, Glorot, Xavier, and Bengio,
  Yoshua.
\newblock Contractive auto-encoders: Explicit invariance during feature
  extraction.
\newblock In \emph{Proceedings of the 28th international conference on machine
  learning (ICML-11)}, pp.\  833--840, 2011.

\bibitem[Russakovsky et~al.(2015)Russakovsky, Deng, Su, Krause, Satheesh, Ma,
  Huang, Karpathy, Khosla, Bernstein, Berg, and Fei-Fei]{ILSVRC15}
Russakovsky, Olga, Deng, Jia, Su, Hao, Krause, Jonathan, Satheesh, Sanjeev, Ma,
  Sean, Huang, Zhiheng, Karpathy, Andrej, Khosla, Aditya, Bernstein, Michael,
  Berg, Alexander~C., and Fei-Fei, Li.
\newblock {ImageNet Large Scale Visual Recognition Challenge}.
\newblock \emph{International Journal of Computer Vision (IJCV)}, 115\penalty0
  (3):\penalty0 211--252, 2015.
\newblock \doi{10.1007/s11263-015-0816-y}.

\bibitem[Salimans et~al.(2016)Salimans, Goodfellow, Zaremba, Cheung, Radford,
  and Chen]{salimans2016improved}
Salimans, Tim, Goodfellow, Ian, Zaremba, Wojciech, Cheung, Vicki, Radford,
  Alec, and Chen, Xi.
\newblock Improved techniques for training gans.
\newblock \emph{arXiv preprint arXiv:1606.03498}, 2016.

\bibitem[Vincent et~al.(2010)Vincent, Larochelle, Lajoie, Bengio, and
  Manzagol]{vincent2010stacked}
Vincent, Pascal, Larochelle, Hugo, Lajoie, Isabelle, Bengio, Yoshua, and
  Manzagol, Pierre-Antoine.
\newblock Stacked denoising autoencoders: Learning useful representations in a
  deep network with a local denoising criterion.
\newblock \emph{Journal of Machine Learning Research}, 11\penalty0
  (Dec):\penalty0 3371--3408, 2010.

\bibitem[Vinyals et~al.(2016)Vinyals, Blundell, Lillicrap, Kavukcuoglu, and
  Wierstra]{vinyals2016matching}
Vinyals, Oriol, Blundell, Charles, Lillicrap, Timothy, Kavukcuoglu, Koray, and
  Wierstra, Daan.
\newblock Matching networks for one shot learning.
\newblock \emph{arXiv preprint arXiv:1606.04080}, 2016.

\bibitem[Yu et~al.(2015)Yu, Seff, Zhang, Song, Funkhouser, and
  Xiao]{yu2015lsun}
Yu, Fisher, Seff, Ari, Zhang, Yinda, Song, Shuran, Funkhouser, Thomas, and
  Xiao, Jianxiong.
\newblock Lsun: Construction of a large-scale image dataset using deep learning
  with humans in the loop.
\newblock \emph{arXiv preprint arXiv:1506.03365}, 2015.

\bibitem[Zhao et~al.(2015)Zhao, Mathieu, Goroshin, and Lecun]{zhao2015stacked}
Zhao, Junbo, Mathieu, Michael, Goroshin, Ross, and Lecun, Yann.
\newblock Stacked what-where auto-encoders.
\newblock \emph{arXiv preprint arXiv:1506.02351}, 2015.

\end{thebibliography}
\bibliographystyle{iclr2017_conference}
}

\newpage
\appendix
\section{Appendix: Technical points of section~\ref{sub:optim}}
\label{app:technical}

\begin{lemma}
\label{lem:phi}
Let $a,b \geq 0$, $\varphi(y)=ay + b\left[m-y\right]^+$. The minimum of $\varphi$ on $[0,+\infty)$ exists and is reached in $m$ if $a<b$, and it is reached in $0$ otherwise (the minimum may not be unique).
\end{lemma}

\begin{proof}
The function $\varphi$ is defined on $[0, +\infty)$, its derivative is defined on $[0, +\infty) \backslash \{m\}$ and $\varphi'(y) = a-b$ if $y\in [0, m)$ and $\varphi'(y) = a$ if $y\in (m, +\infty)$. \\
So when $a < b$, the function is decreasing on $[0, m)$ and increasing on $(m, +\infty)$. Since it is continuous, it has a minimum in $m$.
It may not be unique if $a=0$ or $a-b=0$. \\
On the other hand, if $a \geq b$ the function $\varphi$ is  increasing on $[0, +\infty)$, so $0$ is a minimum.
\end{proof}

\begin{lemma}
\label{lem:indicator}
If $p$ and $q$ are probability densities, then $\int_x \mathbbm{1}_{p(x) < q(x)} \ud x = 0$ if and only if $\int_x \mathbbm{1}_{p(x) \neq q(x)} \ud x = 0$.
\end{lemma}

\begin{proof}
Let's assume that $\int_x \mathbbm{1}_{p(x) < q(x)} \ud x = 0$. Then
\begin{eqnarray}
&& \int_x \mathbbm{1}_{p(x) > q(x)} (p(x) - q(x))\ud x \\
& = & \int_x (1 - \mathbbm{1}_{p(x) \leq q(x)}) (p(x) - q(x))\ud x \\
& = & \int_x p(x) \ud x - \int_x q(x) \ud x + \int_x \mathbbm{1}_{p(x) \leq q(x)} (p(x) - q(x))\ud x \\
& = & 1 - 1 + \int_x \left( \mathbbm{1}_{p(x) < q(x)} +\mathbbm{1}_{p(x) = q(x)}\right) ( p(x) - q(x)) \ud x \\
& = &\int_x \mathbbm{1}_{p(x) < q(x)} ( p(x) - q(x)) \ud x  + \int_x \mathbbm{1}_{p(x) = q(x)} ( p(x) - q(x)) \ud x\\
& = & 0 + 0 = 0
\end{eqnarray}
So $\int_x \mathbbm{1}_{p(x) > q(x)} (p(x) - q(x))\ud x = 0$ and since the term in the integral is always non-negative, $\mathbbm{1}_{p(x) > q(x)} (p(x) - q(x)) = 0$ for almost all $x$. And $p(x) - q(x) = 0$ implies $\mathbbm{1}_{p(x) > q(x)} = 0$, so $\mathbbm{1}_{p(x) > q(x)} = 0$ almost everywhere. 
Therefore $\int_x \mathbbm{1}_{p(x) > q(x)} \ud x = 0$ which completes the proof, given the hypothesis.
\end{proof}

\paragraph{Proof of theorem~\ref{theo:charac}}
The sufficient conditions are obvious.
The necessary condition on $G^*$ comes from theorem~\ref{theo:nash}, and the necessary condition on $D^*(x) \leq m$ is from the proof of theorem~\ref{theo:nash}.\\
Let us now assume that $D^*(x)$ is not constant almost everywhere and find a contradiction.
If it is not, then there exists a constant $C$ and a set $\mathcal{S}$ of non-zero measure such that $\forall x \in \mathcal{S}, D^*(x) \leq C$ and
$\forall x \not\in \mathcal{S}, D^*(X) > C$. In addition we can choose $\mathcal{S}$ such that there exists a subset $\mathcal{S'}\subset \mathcal{S}$ of non-zero measure such that $p_{data}(x)>0$ on $\mathcal{S'}$ (because of the assumption in the footnote). We can build a generator $G_0$ such that $p_{G_0}(x)\leq p_{data}(x)$ over $\mathcal{S}$ and $p_{G_0}(x) < p_{data}(x)$ over $\mathcal{S'}$. We compute
\begin{align}
U(G^*,D^*)-U(G_0,D^*)
&= \int_x (p_{data}-p_{G_0})D^*(x)\ud x \\
&=\int_x(p_{data}-p_{G_0})(D^*(x)-C)\ud x \\
&=\int_\mathcal{S}(p_{data}-p_{G_0})(D^*(x)-C)\ud x +\int_{\mathcal{R}^N\backslash\mathcal{S}}(p_{data}-p_{G_0})(D^*(x)-C)\ud x \\
&> 0
\end{align}
which violates equation \ref{nash:U}.

\newpage
\section{Appendix: More interpretations about GANs and energy-based learning}
\label{app:more}
\subsection*{Two interpretations of GANs}
GANs can be interpreted in two complementary ways. In the first interpretation, the key component is the generator, and the discriminator plays the role of a trainable objective function. Let us imagine that the data lies on a manifold. Until the generator produces samples that are recognized as being on the manifold, it gets a gradient indicating how to modify its output so it could approach the manifold. 
In such scenario, the discriminator acts to punish the generator when it produces samples that are outside the manifold. 
This can be understood as a way to train the generator with a set of possible desired outputs (e.g. the manifold) instead of a single desired output as in traditional supervised learning.

For the second interpretation, the key component is the discriminator, and the generator is merely trained to produce contrastive samples. 
We show that by iteratively and interactively feeding contrastive samples, the generator enhances the semi-supervised learning performance of the discriminator (e.g. Ladder Network), in section \ref{sub:semi}.

\section{Appendix: Experiment settings}
\label{app:pt}

\subsection*{More details about the grid search}
\label{app:grid}
For training both EBGANs and GANs for the grid search, we use the following setting:
\begin{tightitemizeleft}
\item Batch normalization \citep{ioffe2015batch} is applied after each weight layer, except for the generator output layer and the discriminator input layer \citep{radford2015unsupervised}.
\item Training images are scaled into range [-1,1]. Correspondingly the generator output layer is followed by a \texttt{Tanh} function.
\item \texttt{ReLU} is used as the non-linearity function.
\item Initialization: the weights in $D$ from $\mathcal{N}(0, 0.002)$ and in $G$ from $\mathcal{N}(0, 0.02)$. The bias are initialized to be $0$.
\end{tightitemizeleft}

We evaluate the models from the grid search by calculating a modified version of the inception score, $I' = E_x KL(p(y)||p(y|\vec{x}))$, where $\vec{x}$ denotes a generated sample and $y$ is the label predicted by a MNIST classifier that is trained off-line using the entire MNIST training set.
Two main changes were made upon its original form: (i)-we swap the order of the distribution pair; (ii)-we omit the $e^{(\cdot)}$ operation.
The modified score condenses the histogram in figure \ref{fig:hist_all} and figure \ref{fig:hist_sep_all}.
It is also worth noting that although we inherit the name ``inception score'' from \cite{salimans2016improved}, the evaluation isn't related to the ``inception'' model trained on ImageNet dataset. The classifier is a regular 3-layer ConvNet trained on MNIST.

The generations showed in figure \ref{fig:mnist_gen} are the best GAN or EBGAN (obtaining the best $I'$ score) from the grid search. Their configurations are:
\begin{tightitemizeleft}
\item figure \ref{fig:mnist_gen}(a): 
\texttt{nLayerG}=5, \texttt{nLayerD}=2, \texttt{sizeG}=1600, \texttt{sizeD}=1024, \texttt{dropoutD}=0, \texttt{optimD}=SGD, \texttt{optimG}=SGD, \texttt{lr}=0.01.
\item figure \ref{fig:mnist_gen}(b):
\texttt{nLayerG}=5, \texttt{nLayerD}=2, \texttt{sizeG}=800, \texttt{sizeD}=1024, \texttt{dropoutD}=0, \texttt{optimD}=ADAM, \texttt{optimG}=ADAM, \texttt{lr}=0.001, \texttt{margin}=10.
\item figure \ref{fig:mnist_gen}(c): same as (b), with $\lambda_{PT}=0.1$.
\end{tightitemizeleft}

\subsection*{LSUN \& CelebA}
\label{app:lsun}
We use a deep convolutional generator analogous to DCGAN's and a deep convolutional auto-encoder for the discriminator. The auto-encoder is composed of strided convolution modules in the feedforward pathway and fractional-strided convolution modules in the feedback pathway. 
We leave the usage of upsampling or switches-unpooling \citep{zhao2015stacked} to future research. We also followed the guidance suggested by \cite{radford2015unsupervised} for training EBGANs.
The configuration of the deep auto-encoder is:
\begin{tightitemize}
\item Encoder: \texttt{(64)4c2s-(128)4c2s-(256)4c2s}
\item Decoder: \texttt{(128)4c2s-(64)4c2s-(3)4c2s}
\end{tightitemize}
where ``{\small \texttt{(64)4c2s}}'' denotes a convolution/deconvolution layer with 64 output feature maps and kernel size 4 with stride 2. The margin $m$ is set to $80$ for LSUN and $20$ for CelebA.

\subsection*{ImageNet}
\label{app:imagenet}
We built deeper models in both $128\times128$ and $256\times256$ experiments, in a similar fashion to section \ref{sub:lsun},
{\small
\begin{tightitemize}
\item $128 \times 128$ model:
    \begin{tightitemize}
	    \item Generator: \texttt{(1024)4c-(512)4c2s-(256)4c2s-(128)4c2s-\\(64)4c2s-(64)4c2s-(3)3c}
        \item Noise \#planes: \texttt{100-64-32-16-8-4}
		\item Encoder: \texttt{(64)4c2s-(128)4c2s-(256)4c2s-(512)4c2s}
		\item Decoder: \texttt{(256)4c2s-(128)4c2s-(64)4c2s-(3)4c2s}
        \item Margin: $40$
	\end{tightitemize}

\item $256 \times 256$ model:
    \begin{tightitemize}
	    \item Generator: \texttt{(2048)4c-(1024)4c2s-(512)4c2s-(256)4c2s-(128)4c2s-\\(64)4c2s-(64)4c2s-(3)3c}
        \item Noise \#planes: \texttt{100-64-32-16-8-4-2}
		\item Encoder: \texttt{(64)4c2s-(128)4c2s-(256)4c2s-(512)4c2s}
		\item Decoder: \texttt{(256)4c2s-(128)4c2s-(64)4c2s-(3)4c2s}
        \item Margin: $80$
	\end{tightitemize}
\end{tightitemize}
}
Note that we feed noise into every layer of the generator where each noise component is initialized into a 4D tensor and concatenated with current feature maps in the feature space. Such strategy is also employed by \cite{salimans2016improved}.

\section{Appendix: Semi-supervised learning experiment setting}
\label{app:semi}
\subsection*{Baseline model}
As stated in section \ref{sub:semi}, we chose a bottom-layer-cost Ladder Network as our baseline model. 
Specifically, we utilize an identical architecture as reported in both papers \citep{rasmus2015semi,pezeshki2015deconstructing}; namely a fully-connected network of size \texttt{784-1000-500-250-250-250}, with batch normalization and ReLU following each linear layer.
To obtain a strong baseline, we tuned the weight of the reconstruction cost with values from the set \{$\frac{5000}{784}$, $\frac{2000}{784}$, $\frac{1000}{784}$, $\frac{500}{784}$\}, while fixing the weight on the classification cost to $1$. In the meantime, we also tuned the learning rate with values \{0.002, 0.001, 0.0005, 0.0002, 0.0001\}. We adopted Adam as the optimizer with $\beta_1$ being set to 0.5. The minibatch size was set to 100. All the experiments are finished by 120,000 steps. 
We use the same learning rate decay mechanism as in the published papers -- starting from the two-thirds of total steps (i.e., from step \#80,000) to linearly decay the learning rate to $0$.
The result reported in section \ref{sub:semi} was done by the best tuned setting: $\lambda_{L2}=\frac{1000}{784}, lr=0.0002$.

\subsection*{EBGAN-LN model}
We place the same Ladder Network architecture into our EBGAN framework and train this EBGAN-LN model the same way as we train the EBGAN auto-encoder model. For technical details, we started training the EBGAN-LN model from the margin value 16 and gradually decay it to 0 within the first 60,000 steps. By the time, we found that the reconstruction error of the real image had already been low and reached the limitation of the architecture (Ladder Network itself); besides the generated images exhibit good quality (shown in figure \ref{fig:mnist_ln}). Thereafter we turned off training the generator but kept training the discriminator for another 120,000 steps. We set the initial learning rates to be $0.0005$ for discriminator and $0.00025$ for generator.
The other setting is kept consistent with the best baseline LN model. The learning rate decay started at step \#120,000 (also two-thirds of the total steps).

\subsection*{Other details}
\begin{tightitemizeleft}
\item Notice that we used the 28$\times$28 version (unpadded) of the MNIST dataset in the EBGAN-LN experiment. For the EBGAN auto-encoder grid search experiments, we used the zero-padded version, i.e., size 32$\times$32. No phenomenal difference has been found due to the zero-padding.
\item We generally took the $\ell_2$ norm of the discrepancy between input and reconstruction for the loss term in the EBGAN auto-encoder model as formally written in section \ref{sub:obj}. However, for the EBGAN-LN experiment, we followed the original implementation of Ladder Network using a vanilla form of $\ell_2$ loss. 
\item Borrowed from \cite{salimans2016improved}, the batch normalization is adopted without the learned parameter $\gamma$ but merely with a bias term $\beta$. It still remains unknown whether such trick could affect learning in some non-ignorable way, so this might have made our baseline model not a strict reproduction of the published models by \cite{rasmus2015semi} and \cite{pezeshki2015deconstructing}.
\end{tightitemizeleft}

\section{Appendix: tips for setting a good energy margin value}
It is crucial to set a proper energy margin value $m$ in the framework of EBGAN, from both theoretical and experimental perspective.
Hereby we provide a few tips:
\begin{tightitemize}
\item Delving into the formulation of the discriminator loss made by equation \ref{discloss}, we suggest a numerical balance between its two terms which concern \emph{real} and \emph{fake} sample respectively.
The second term is apparently bounded by $[0, m]$ (assuming the energy function $D(x)$ is non-negative). It is desirable to make the first term bounded in a similar range.
In theory, the upper bound of the first term is essentially determined by (i)-the capacity of $D$; (ii)-the complexity of the dataset.
\item In practice, for the EBGAN auto-encoder model, one can run $D$ (the auto-encoder) alone on the real sample dataset and monitor the loss. When it converges, the consequential loss implies a rough limit on how well such setting of $D$ is capable to fit the dataset. This usually suggests a good start for a hyper-parameter searching on $m$.
\item $m$ being overly large results in a training instability/difficulty, while $m$ being too small is prone to the mode-dropping problem. This property of $m$ is depicted in figure \ref{fig:mnist_margin}.
\item One successful technique, as we introduced in appendix \ref{app:semi}, is to start from a large $m$ and gradually decayed it to 0 along training proceeds.
Unlike the feature matching semi-supervised learning technique proposed by \cite{salimans2016improved}, we show in figure \ref{fig:mnist_ln} that not only does the EBGAN-LN model achieve a good semi-supervised learning performance, it also produces satisfactory generations.
\end{tightitemize}

Abstracting away from the practical experimental tips, the theoretical understanding of EBGAN in section \ref{sub:optim} also provides some insight for setting a feasible $m$.
For instance, as implied by Theorem \ref{theo:charac}, setting a large $m$ results in a broader range of $\gamma$ to which $D^*(x)$ may converge. Instability may come after an overly large $\gamma$ because it generates two strong gradients pointing to opposite directions, from loss \ref{discloss}, which would demand more finicky optimization setting.

\vspace{2mm} 
\begin{figure}[ht]
\centering
\begin{tabular}{c}
\includegraphics[width=0.7\linewidth]{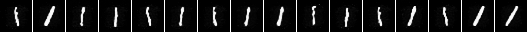}\\
\includegraphics[width=0.7\linewidth]{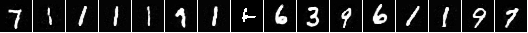}\\
\includegraphics[width=0.7\linewidth]{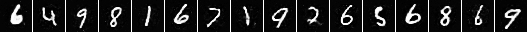}\\
\includegraphics[width=0.7\linewidth]{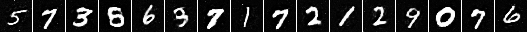}\\
\includegraphics[width=0.7\linewidth]{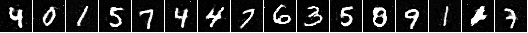}\\
\includegraphics[width=0.7\linewidth]{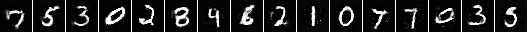}\\
\includegraphics[width=0.7\linewidth]{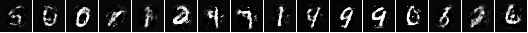}\\
\includegraphics[width=0.7\linewidth]{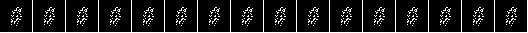}
\end{tabular}
\caption{Generation from the EBGAN auto-encoder model trained with different $m$ settings. From top to bottom, $m$ is set to 1, 2, 4, 6, 8, 12, 16, 32 respectively. The rest setting is \texttt{nLayerG}=5, \texttt{nLayerD}=2, \texttt{sizeG}=1600, \texttt{sizeD}=1024, \texttt{dropoutD}=0, \texttt{optimD}=ADAM, \texttt{optimG}=ADAM, \texttt{lr}=0.001.}
\label{fig:mnist_margin}
\end{figure}

\begin{figure}[ht]
\centering
\begin{tabular}{cc}
\includegraphics[width=0.3\linewidth]{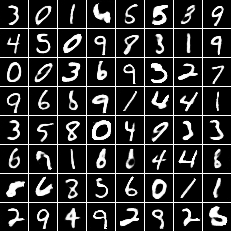} & \includegraphics[width=0.3\linewidth]{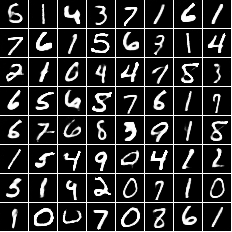}
\end{tabular}
\caption{Generation from the EBGAN-LN model. The displayed generations are obtained by an identical experimental setting described in appendix \ref{app:semi}, with different random seeds. As we mentioned before, we used the unpadded version of the MNIST dataset (size 28$\times$28) in the EBGAN-LN experiments.}
\label{fig:mnist_ln}
\end{figure}

\section{Appendix: more generation}
\subsection*{LSUN augmented version training}
\label{app:auglsun}
For LSUN bedroom dataset, aside from the experiment on the whole images, we also train an EBGAN auto-encoder model based on dataset augmentation by cropping patches. All the patches are of size $64 \times 64$ and cropped from $96 \times 96$ original images. The generation is shown in figure \ref{fig:lsun_aug_gen}.

\begin{figure}[ht]
\centering
\minipage{0.48\textwidth}
\includegraphics[width=\linewidth]{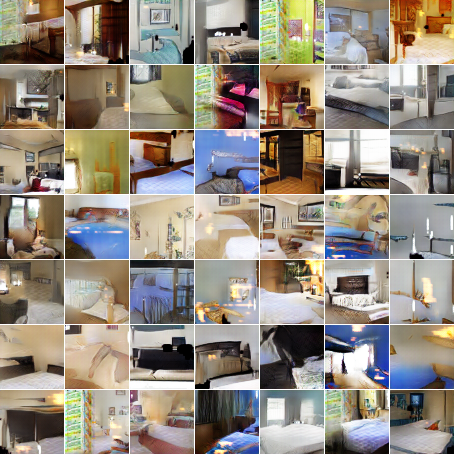}
\endminipage \hspace{5pt}
\minipage{0.48\textwidth}
\includegraphics[width=\linewidth]{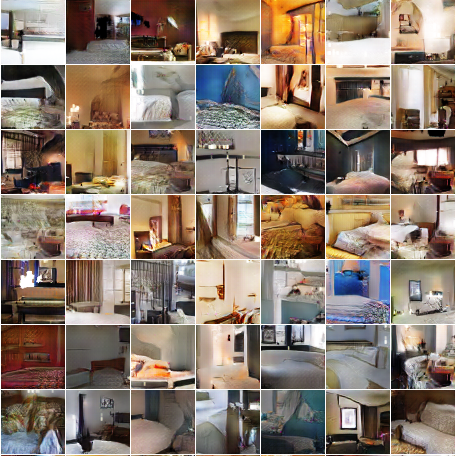}
\endminipage
\caption{Generation from augmented-patch version of the LSUN bedroom dataset. Left(a): DCGAN generation. Right(b): EBGAN-PT generation.}
\label{fig:lsun_aug_gen}
\end{figure}

\subsection*{Comparison of EBGANs and EBGAN-PTs}
To further demonstrate how the pull-away term (PT) may influence EBGAN auto-encoder model training, we chose both the whole-image and augmented-patch version of the LSUN bedroom dataset, together with the CelebA dataset to make some further experimentation.
The comparison of EBGAN and EBGAN-PT generation are showed in figure \ref{fig:lsun_full_gen_pt}, figure \ref{fig:lsun_aug_gen_pt} and figure \ref{fig:celeba_gen_pt}.
Note that all comparison pairs adopt identical architectural and hyper-parameter setting as in section \ref{sub:lsun}. The cost weight on the PT is set to $0.1$.

\begin{figure}[ht]
\centering
\minipage{0.47\textwidth}
\includegraphics[width=\linewidth]{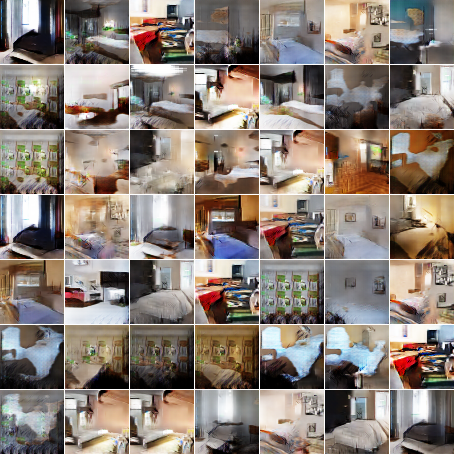}
\endminipage \hspace{5pt}
\minipage{0.47\textwidth}
\includegraphics[width=\linewidth]{./lsun_full_dot.png}
\endminipage
\caption{Generation from whole-image version of the LSUN bedroom dataset. Left(a): EBGAN. Right(b): EBGAN-PT.}
\label{fig:lsun_full_gen_pt}
\end{figure}

\begin{figure}[ht]
\centering
\minipage{0.47\textwidth}
\includegraphics[width=\linewidth]{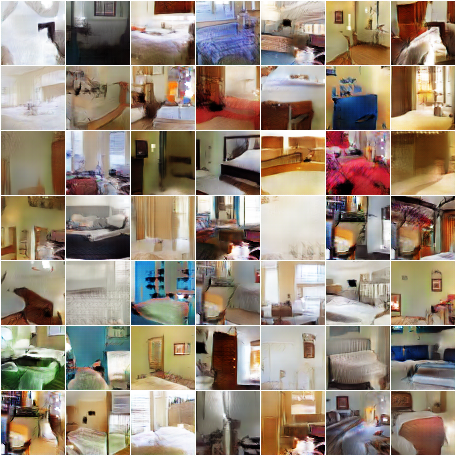}
\endminipage \hspace{5pt}
\minipage{0.47\textwidth}
\includegraphics[width=\linewidth]{./lsun_aug_dot.png}
\endminipage
\caption{Generation from augmented-patch version of the LSUN bedroom dataset. Left(a): EBGAN. Right(b): EBGAN-PT.}
\label{fig:lsun_aug_gen_pt}
\end{figure}

\begin{figure}[ht]
\centering
\minipage{0.47\textwidth}
\includegraphics[width=\linewidth]{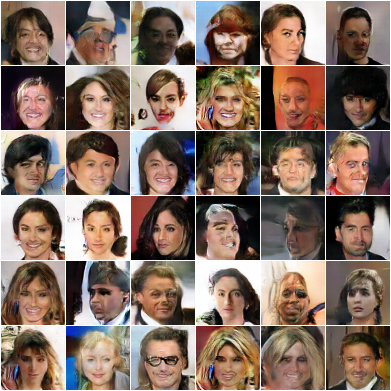}
\endminipage \hspace{5pt}
\minipage{0.47\textwidth}
\includegraphics[width=\linewidth]{./celeba_dot.png}
\endminipage
\caption{Generation from the CelebA dataset. Left(a): EBGAN. Right(b): EBGAN-PT.}
\label{fig:celeba_gen_pt}
\end{figure}

\end{document}